\documentclass[runningheads]{llncs}
\usepackage[T1]{fontenc}
\usepackage{graphicx}
\usepackage{paralist}
\usepackage{subcaption}

\usepackage{amsmath,nccmath}
\usepackage{amssymb}
\usepackage{xspace}
\usepackage{xcolor}

\newcommand{\myi}{(\emph{i})\xspace}
\newcommand{\myii}{(\emph{ii})\xspace}
\newcommand{\myiii}{(\emph{iii})\xspace}
\newcommand{\myiv}{(\emph{iv})\xspace}

\newcommand{\mya}{(\emph{a})\xspace}
\newcommand{\myb}{(\emph{b})\xspace}
\newcommand{\myc}{(\emph{c})\xspace}
\newcommand{\A}{\mathcal{A}}

\newcommand{\D}{\mathcal{D}}
\newcommand{\E}{\mathcal{E}} 

\newcommand{\G}{\mathcal{G}}

\renewcommand{\L}{\mathcal{L}}

\renewcommand{\P}{\mathcal{P}}

\newcommand{\U}{\mathcal{U}}

\newcommand{\X}{\mathcal{X}}
\newcommand{\Y}{\mathcal{Y}} 
\newcommand{\Z}{\mathcal{Z}}

\newcommand{\trace}{\pi}
\newcommand{\tiff}{\text{ iff }}
\newcommand{\LTL}{{\sc ltl}\xspace}
\newcommand{\LTLf}{{\sc ltl}$_f$\xspace}

\newcommand{\DFA}{{\sc dfa}\xspace}
\newcommand{\DFAs}{{\sc dfa}s\xspace}

\newcommand{\true}{\mathit{true}}

\newcommand{\false}{\mathit{false}}

\newcommand{\Next}{\raisebox{-0.27ex}{\LARGE$\circ$}}
\newcommand{\Wnext}{\raisebox{-0.27ex}{\LARGE$\bullet$}}
\newcommand{\Until}{\mathop{\U}}

\newcommand{\lneg}{\neg}
\newcommand{\limpl}{\rightarrow}

\newcommand{\WinSet}{\mathsf{W}}

\newcommand\run[1]{{\sf{Run}{#1}}}

\newcommand{\ls}{{\sf{lst}}}

\newcommand{\env}{ {\sc env} \xspace}
\newcommand{\ag}{ {\sc ag} \xspace}

\newcommand{\stagd}{{\sigma}_{\ag}}

\newcommand{\wine}{\mathsf{W}_\env}
\newcommand{\wina}{\mathsf{W}_\ag}

\newcommand{\besyft}{\textit{BeSyft}\xspace}

\newcommand{\comp}{\mathsf{Comp}\xspace}
\newcommand{\inter}{\mathsf{Inter}\xspace}

\begin{document}
%
\title{Symbolic \LTLf Best-Effort Synthesis}

%
%
\author{Giuseppe De Giacomo\inst{1,2} \and
Gianmarco Parretti\inst{1}\thanks{Corresponding author} \and
Shufang Zhu\inst{2}\thanks{Corresponding author}}
\authorrunning{De Giacomo et al.}
%
\institute{University of Rome ``La Sapienza”, Italy \and
University of Oxford, UK\\
\email{\{degiacomo, parretti\}@diag.uniroma1.com} \\
\email{giuseppe.degiacomo@cs.ox.ac.uk} \\
\email{shufang.zhu@cs.ox.ac.uk}}
\maketitle

\begin{abstract}
We consider an agent acting to fulfil tasks in a nondeterministic environment. When a strategy that fulfills the task regardless of how the environment acts does not exist, the agent should at least avoid adopting strategies that prevent from fulfilling its task. Best-effort synthesis captures this intuition.
In this paper, we devise and compare various symbolic approaches for best-effort synthesis in Linear Temporal Logic on finite traces (\LTLf).  These approaches are based on the same basic components, however they change in how these components are combined, and this has a significant impact on the performance of the approaches as confirmed by our empirical evaluations. 

\end{abstract}

\section{Introduction}\label{sec:intro}
We consider an agent acting to fulfill tasks in a nondeterministic environment, as considered in Planning in nondeterministic~(adversarial) domains~\cite{Cimatti03,ghallab2004automated}, except that we specify both the environment and the task in Linear Temporal Logic (\LTL)~\cite{ADMR19}, the formalism typically used for specifying complex dynamic properties in Formal Methods~\cite{BaKG08}.  

In fact, we consider Linear Temporal Logic on finite traces~(\LTLf)~\cite{DegVa13,DegVa15}, which maintains the syntax of \LTL \cite{Pnu77} but is interpreted on finite traces. In this setting, we study synthesis~\cite{PR89,finkbeiner2016synthesis,DegVa15,ADMR19}. In particular, we look at how to synthesize a strategy that is guaranteed to satisfy the task against all environment behaviors that conform to the environment specification.

When a winning strategy that fulfills the agent's task, regardless of how the environment acts, does not exist, the agent should at least avoid adopting strategies that prevent it from fulfilling its task. Best-effort synthesis captures this intuition. More precisely, best-effort synthesis captures the game-theoretic rationality principle that a player would not use a strategy that is ``dominated'' by another of its strategies~(i.e. if the other strategy would fulfill the task against more environment behaviors than the one chosen by the player). 
Best-effort strategies have been studied in~\cite{ADR2021} and proven to have some notable properties:  \myi they always exist, \myii if a winning strategy exists, then best-effort strategies are exactly the winning strategies, \myiii best-effort strategies can be computed in 2EXPTIME as computing winning strategies (best-effort synthesis is indeed 2EXPTIME-complete).

The algorithms for best-effort synthesis in \LTL and \LTLf have been presented in~\cite{ADR2021}. These algorithms are based on creating, solving, and combining the solutions of three distinct games but of the same game arena. 
The arena is obtained from the automata corresponding to the formulas $\E$ and $\varphi$ constituting the environment and the task specifications, respectively.

In particular, the algorithm for \LTLf best-effort synthesis appears to be quite promising in practice since well-performing techniques for each component of the algorithm are available in the literature. These components are: \myi transformation of the \LTLf formulas $\E$ and $\varphi$ into deterministic finite automata (\DFA), which can be double-exponential in the worst case, but for which various good techniques have been developed~\cite{mona,ZTLPV17,BLTV,lydia}; \myii Cartesian product of \DFAs, which is polynomial; \myiii minimization of \DFAs, which is also polynomial;  \myiv fixpoint computation over \DFA to compute adversarial and cooperative winning strategies for reaching the final states, which is again polynomial. 

In this paper, we refine the \LTLf best-effort synthesis techniques presented in~\cite{ADR2021} by using symbolic techniques~\cite{Bryant92,BaKG08,ZTLPV17}. In particular, we show three different symbolic approaches that combine the above operations in different ways (and in fact allow for different levels of minimization). We then compare the three approaches through empirical evaluations. 
From this comparison, a clear winner emerges. Interestingly, the winner does not fully exploit \DFA minimization to minimize the \DFA whenever it is possible. Instead, this approach uses uniformly the same arena for all three games (hence giving up on minimization at some level). Finally, it turns out that the winner performs better in computing best-effort solutions even than state-of-the-art tools that compute only adversarial solutions. These findings confirm that \LTLf best-effort synthesis is indeed well suited for efficient and scalable implementations.

The rest of the paper is organized as follows. In Section~\ref{sec:2}, we recall the main notions of \LTLf synthesis.
In Section~\ref{sec:3}, we discuss \LTLf best-effort synthesis, and the algorithm presented in~\cite{ADR2021}. In Section~\ref{sec:4}, we introduce three distinct symbolic approaches for \LTLf best-effort synthesis: the first (c.f., Subsection~\ref{sec:4-2}) is a direct symbolic implementation of the algorithm presented in \cite{ADR2021}; the second one (c.f., Subsection~\ref{sec:4-3}) favors maximally conducting \DFA minimization, thus getting the smallest possible arenas for the three games; and the third one (c.f., Subsection~\ref{sec:4-4}) gives up \DFA minimization at some level, and creates a single arena for the three games. In Section~\ref{sec:5}, we perform an empirical evaluation of the three algorithms. We conclude the paper in Section~\ref{sec:6}.

\section{Preliminaries}\label{sec:2}
\paragraph{\LTLf Basics.} 
\textit{Linear Temporal Logic on finite traces}~(\LTLf) is a specification language to express temporal properties on finite traces~\cite{DegVa13}. In particular, \LTLf
has the same syntax as \LTL, which is instead interpreted over infinite traces~\cite{Pnu77}. Given a set of propositions $\Sigma$, \LTLf formulas are generated as follows: 
\[
\varphi ::= a \mid (\varphi_1 \wedge \varphi_2) \mid (\neg \varphi) \mid  
 (\Next \varphi) \mid (\varphi_1 \Until \varphi_2)
 \]
where $a \in \Sigma$ is an \textit{atom}, $\Next$~(\emph{Next}), and $\Until$~(\emph{Until}) are temporal operators. 
We make use of standard Boolean abbreviations such as $\vee$~(or) and $\limpl$~(implies), $\true$ and $\false$. In addition, we define the following abbreviations \emph{Weak Next} $\Wnext \varphi \equiv \neg \Next \neg \varphi$, \emph{Eventually} $\Diamond \varphi \equiv \true \Until \varphi$ and \emph{Always} $\Box \varphi \equiv \lneg \Diamond \lneg \varphi$.
The length/size of $\varphi$, written $|\varphi|$, is the number of operators in $\varphi$.

A \emph{finite} (resp. \emph{infinite}) \emph{trace}  is a sequence of propositional interpretations $\trace \in (2^{\Sigma})^*$ (resp. $\pi \in (2^{\Sigma})^\omega$). For every $i \geq 0$, $\trace_i \in 2^{\Sigma}$ is the $i$-th interpretation of $\trace$. Given a finite trace $\trace$, we denote its last instant~(i.e., index) by $\ls(\trace)$. 
\LTLf formulas are interpreted over finite, nonempty traces.  Given a finite, non-empty trace $\trace \in (2^\Sigma)^+$, we define when an \LTLf formula $\varphi$ \emph{holds} at instant $i$, $0 \leq i \leq \ls(\pi)$, written $\trace, i \models \varphi$, inductively on the structure of $\varphi$, as:
\begin{compactitem}
	\item 
	$\trace, i \models a \tiff a \in \trace_i\nonumber$ (for $a\in{\Sigma}$);
	\item 
	$\trace, i \models \lnot \varphi \tiff \trace, i \not\models \varphi\nonumber$;
	\item 
	$\trace, i \models \varphi_1 \wedge \varphi_2 \tiff \trace, i \models \varphi_1 \text{ and } \trace, i \models \varphi_2\nonumber$;
	\item 
	$\pi, i \models \Next\varphi \tiff  i< \ls(\pi)$ and $\trace,i+1 \models \varphi$;
	\item 
	$\pi, i \models \varphi_1 \Until \varphi_2$ iff $\exists j$ such that $i \leq j \leq \ls(\pi)$ and $\pi,j \models\varphi_2$, and $\forall k, i\le k < j$ we have that $\pi, k \models \varphi_1$.
\end{compactitem}
We say $\pi$ \emph{satisfies} $\varphi$, written as $\pi \models \varphi$, if $\pi, 0 \models \varphi$. 


\paragraph{Reactive Synthesis Under Environment Specifications.} 
Reactive synthesis concerns computing a strategy that allows the agent to achieve its goal in an adversarial environment. In many AI applications, the agent has a model describing possible environment behaviors, which we call here an \emph{environment specification}~\cite{ADMR18,ADMR19}. In this work, we specify both environment specifications and agent goals as \LTLf formulas defined over $\Sigma = \X \cup \Y$, where $\X$ and $\Y$ are disjoint sets of variables under the control of the environment and the agent, respectively. 



An \emph{agent strategy} is a function \(\sigma_{ag}: (2^{\X})^* \rightarrow 2^{\Y}\) that maps a sequence of environment choices to an agent choice. Similarly, an \textit{environment strategy} is a function \(\sigma_{env}: (2^{\Y})^+ \rightarrow 2^\X\) mapping non-empty sequences of agent choices to an environment choice. 
A trace $\pi = (X_0 \cup Y_0)(X_1 \cup Y_1)\ldots  \in (2^{\X \cup \Y})^\omega$ is $\sigma_{ag}$-consistent if $Y_0 = \sigma_{ag}(\epsilon)$, where $\epsilon$ denotes empty sequence, and
$Y_{i} = \sigma_{ag}(X_0, \ldots, X_{i-1})$ for every $i > 0$. Analogously, $\pi$ is $\sigma_{env}$-consistent if $X_i = \sigma_{env}(Y_0, \ldots, Y_i)$ for every $i \geq 0$. We define $\pi(\sigma_{ag}, \sigma_{env})$ to be the unique infinite trace that is consistent with both $\sigma_{ag}$ and $\sigma_{env}$. 

Let $\psi$ be an \LTLf formula over $\X \cup \Y$. We say that agent strategy \(\sigma_{ag}\) \emph{enforces} \(\psi\), written \(\sigma_{ag} \triangleright \psi\), if for every environment strategy $\sigma_{env}$, there exists a \emph{finite} prefix of $\pi(\sigma_{ag}, \sigma_{env})$ that satisfies $\psi$.
Conversely, we say that an environment strategy $\sigma_{env}$ \emph{enforces} $\psi$, written $\sigma_{env} \triangleright \psi$, if for every agent strategy $\sigma_{ag}$, every finite prefix of $\pi(\sigma_{ag}, \sigma_{env})$ satisfies $\psi$.
$\psi$ is \emph{agent enforceable} (resp. \emph{environment enforceable}) if there exists an agent  (resp. environment) strategy that enforces it. An \emph{environment specification} is an \LTLf formula $\E$ that is environment enforceable.

%

The problem of \LTLf reactive synthesis under environment specifications is defined as follows.
\begin{definition}~\label{def:reactive}
The \LTLf reactive synthesis under environment specifications problem is defined as a pair \(\P =(\E, \varphi)\), where \LTLf formulas \(\E\) and \(\varphi\) correspond to an environment specification and an agent goal, respectively. Realizability of \(\P\) checks whether there exists an agent strategy \(\sigma_{ag}\) that enforces \(\varphi\) under \(\E\), i.e.,

\[\forall \sigma_{env} \triangleright \E, \pi(\sigma_{ag}, \sigma_{env}) \models \varphi\]
Synthesis of $\P$ computes such a strategy if it exists.

\end{definition}
A naive approach to this problem is a
reduction to standard reactive synthesis of \LTLf formula \(\E \limpl \varphi\)~\cite{ADMR19}. Moreover, it has been shown that the problem of \LTLf reactive synthesis under environment specifications is 2EXPTIME-complete~\cite{ADMR19}. 

\section{Best-effort Synthesis Under Environment Specifications}\label{sec:3}

In reactive synthesis, the agent aims at computing a strategy that enforces the goal regardless of environment behaviors. If such a strategy does not exist, the agent just gives up when the synthesis procedure declares the problem \emph{unrealizable}, although the environment can be possibly ``over-approximated". In this work, we synthesize a strategy ensuring that the agent will do nothing that would needlessly prevent it from achieving its goal -- which we call a \emph{best-effort strategy}.
%
\emph{Best-effort synthesis} is the problem of finding such a strategy~\cite{ADR2021}. 
We start by reviewing what it means for an agent strategy to make more effort with respect to another. 

\begin{definition}~\label{def:dominance}
Let \(\E\) and \(\varphi\) be \LTLf formulas denoting an environment specification and an agent goal, respectively, and \(\sigma_1\) and \(\sigma_2\) be two agent strategies. \(\sigma_1\) \emph{dominates} \(\sigma_2\) for \(\varphi\) under \(\E\), written $\sigma_1 \geq_{\varphi|\E} \sigma_2$, if for every \(\sigma_{env} \triangleright \E, \pi(\sigma_{2}, \sigma_{env}) \models \varphi\) implies \(\pi(\sigma_{1}, \sigma_{env}) \models \varphi\).    
\end{definition}

 Furthermore, we say that \(\sigma_1\) \emph{strictly dominates} \(\sigma_2\), written \(\sigma_ 1 >_{\varphi|\E} \sigma_2\), if \(\sigma_{1} \geq_{\varphi|\E} \sigma_{2}\) and \(\sigma_{2} \not \geq_{\varphi|\E} \sigma_{1}\). 
Intuitively, \(\sigma_1 >_{\varphi | \E} \sigma_2\) means that \(\sigma_1\) does at least as well as \(\sigma_2\) against every environment strategy enforcing \(\E\) and strictly better against one such strategy. If \(\sigma_1\) strictly dominates \(\sigma_2\), then \(\sigma_1\) makes more effort than \(\sigma_{2}\) to satisfy the goal. In other words, if \(\sigma_2\) is strictly dominated by \(\sigma_1\), then an agent that uses \(\sigma_2\) does not do its best to achieve the goal: if it used \(\sigma_1\) instead, it could achieve the goal against a strictly larger set of environment behaviors. Within this framework, a best-effort strategy is one that is not strictly dominated by any other strategy.


\begin{definition}~\label{def:best-effort-strategy}
    An agent strategy \(\sigma\) is best-effort for \(\varphi\) under \(\E\), if there is no agent strategy \(\sigma'\) such that \(\sigma' >_{\varphi|\E} \sigma\).
\end{definition}

It follows immediately from Definition~\ref{def:best-effort-strategy} that if a goal $\varphi$ is agent enforceable under $\E$, then best-effort strategies enforce $\varphi$ under $\E$. Best-effort synthesis concerns computing a best-effort strategy.

\begin{definition}[\cite{ADR2021}]~\label{def:best-effort-synthesis}
The \LTLf best-effort synthesis problem is defined as a pair $\P = (\E, \varphi)$, where \LTLf formulas \(\E\) and \(\varphi\) are the environment specification and the agent goal, respectively. Best-effort synthesis of $\P$ computes an agent strategy that is best-effort for $\varphi$ under $\E$.
\end{definition}

While classical synthesis settings first require checking the realizability of the problem, i.e., the existence of a strategy that enforces the agent goal under environment specification~\cite{DegVa15,PR89}, deciding whether a best-effort strategy exists is trivial, as they always exist.

\begin{theorem}[\cite{ADR2021}] Let $\P = (\E, \varphi)$
be an \LTLf best-effort synthesis problem. There exists a best-effort strategy for $\varphi$ under $\E$.
\end{theorem}

\LTLf best-effort synthesis can be solved by a reduction to suitable \DFA games and is 2EXPTIME-complete~\cite{ADR2021}.

\paragraph{\DFA Game.}
A \emph{\DFA game} is a two-player game played on a deterministic finite automaton~(\DFA). Formally, a \DFA is defined as a pair $\A = (\D, F)$, where $\D$ is a deterministic transition system such that $\D = (2^{\X \cup \Y}, S, s_0, \delta)$, where $2^{\X \cup \Y}$ is the alphabet, $S$ is the state set, $s_0 \in S$ is the initial state and $\delta\colon S \times 2^{\X \cup \Y} \rightarrow S $ is the deterministic \emph{transition function}, and $F \subseteq S$ is a set of final states. We call $|S|$ the \emph{size} of $\D$. Given a finite word $\pi = (X_0 \cup Y_0)\ldots (X_n \cup Y_n) \in (2^{\X \cup \Y})^+$, running $\pi$ in $\D$ yields the sequence $\rho = s_0 \ldots s_{n+1}$ such that $s_0$ is the initial state of $\D$ and $s_{i+1} = \delta(s_i, X_i \cup Y_i)$ for all $i$. Since the transitions in $\D$ are all deterministic, we denote by $\rho = \run(\pi, \D)$ the unique sequence induced by running $\pi$ on $\D$. We define the \emph{product} of transition systems as follows.

\begin{definition}~\label{def:product}
The product of transition systems $\D_i = (\Sigma, S_i, s_{(0,i)}, \delta_i)$ (with $i = 1, 2$) over the same alphabet is the transition system $\D_1 \times \D_2 = (\Sigma, S, s_0, \delta)$ with: $S = S_1 \times S_2$; $s_0 = (s_{(0,1)}, s_{(0,2)})$; and $\delta((s_1, s_2), x) = (\delta(s_1, x), \delta(s_2, x))$. The product $\D_1 \times \ldots \times \D_n$ is defined analogously for any finite sequence $\D_1, \ldots, \D_n$ of transition systems over the same alphabet.
\end{definition}

A finite word $\pi$ is \emph{accepted} by $\A = (\D, F)$ if the last state of the run it induces is a final state, i.e., $\ls(\rho) \in F$, where $\rho = \run(\pi, \D)$. The language of $\A$, denoted as $\L(\A)$, consists of all words accepted by the automaton. Every \LTLf formula $\varphi$ can be transformed into a \DFA $\A_{\varphi}$ that accepts exactly the traces that satisfy the formula, in other words, $\A_{\varphi}$ \emph{recognizes} $\varphi$.

\begin{theorem}[~\cite{DegVa13}] Given an \LTLf formula over $\Sigma$, we can build a \DFA $\A_{\varphi} = (\D_{\varphi}, F_{\varphi})$ whose size is at most double-exponential in $|\varphi|$ such that $\pi \models \varphi$ iff $\pi \in \L(\A_{\varphi})$.
\end{theorem}

In a \DFA game $(\D, F)$, the transition system $\D$ is also called the \emph{game arena}. Given \(\sigma_{ag}\) and \(\sigma_{env}\) denoting an agent strategy and an environment strategy, respectively, the trace $\pi(\sigma_{ag}, \sigma_{env})$ is called a \emph{play}.
Specifically, a play is \emph{winning} if it contains a finite prefix that is accepted by the \DFA. Intuitively, \DFA games require $F$ to be visited at least once.
An agent strategy $\stagd$ is \emph{winning} in $(\D, F)$ if, for every environment strategy $\sigma_{env}$, it results that \(\pi(\sigma_{ag}, \sigma_{env})\) is winning.
Conversely, an environment strategy \(\sigma_{env}\) is \emph{winning} in the game $(\D, F)$ if, for every agent strategy $\sigma_{ag}$, it results that \(\pi(\sigma_{ag}, \sigma_{env})\) is not winning.
In \emph{\DFA games}, $s \in S$ is a \emph{winning} state for the agent~(resp.~environment) if the agent~(resp.~the environment) has a winning strategy in the game $(\D', F)$, where $\D'=(2^{\X \cup \Y}, S, s, \delta)$, i.e., the same arena $\D$ but with the new initial state $s$. By $\WinSet_\ag(\D, F)$~(resp.~$\WinSet_\env(\D, F)$) we denote the set of all agent~(resp.~environment) winning states. Intuitively, $\WinSet_\ag$ represents the ``agent winning region", from which the agent is able to win the game, no matter how the environment behaves.

We also define cooperatively winning strategies for \DFA games. An agent strategy \(\sigma_{ag}\) is \emph{cooperatively winning} in game \((\D, F)\) if there exists an environment strategy \(\sigma_{env}\) such that \(\pi(\sigma_{ag}, \sigma_{env})\) is winning. 
Hence, \(s \in S\) is a \emph{cooperatively winning state} 
if the agent
has a cooperatively winning strategy in the game \((\D', F)\), where \(\D' = (2^{\X \cup \Y}, S, s_0, \delta)\). 
By $\WinSet_\ag'(\D, F)$~
we denote the set of all agent
cooperative winning states.
%

When the agent makes its choices based only on the current state of the game, we say that it uses a \emph{positional strategy}. Formally, we define an \emph{agent positional strategy} (a.k.a. \emph{memory-less strategy}) as a function $\tau_{ag}: S \rightarrow 2^{\X}$. An agent positional strategy $\tau_{ag}$ \emph{induces} an agent strategy $\sigma_{ag}: (2^{\X})^* \rightarrow 2^{\Y}$ as follows: 
$\sigma_{ag}(\epsilon) = \tau(s_0)$ and, for $i \geq 0$, $\sigma_{ag}(X_0\ldots X_{i}) = \tau_{ag}(s_{i+1})$, where $s_{i+1}$ is the last state in the sequence $\rho = \run(\pi, \D)$, with $\pi$ being the finite sequence played until now, i.e., $\pi = (\sigma_{ag}(\epsilon) \cup X_0)(\sigma_{ag}(X_0) \cup X_1)\ldots(\sigma(X_0\ldots X_{k-1}) \cup X_{k})$.
Similarly, we can define an \emph{environment positional} strategy as a function \(\tau_{env}\colon S \times 2^{\Y} \rightarrow 2^\X\). A positional strategy for a player that is winning (resp. cooperatively winning) from every state in its winning region is called \emph{uniform winning} (resp. \emph{uniform cooperatively winning}).

The solution to \LTLf best-effort synthesis presented in \cite{ADR2021} can be summarized as follows.
\newline

\noindent\textbf{Algorithm 0}~\cite{ADR2021}\textbf{.} Given an \LTLf best-effort synthesis problem $\P = (\E, \varphi)$, proceed as follows:
\begin{enumerate}
\item 
For every $\xi \in \{\neg \E, \E \rightarrow \varphi, \E \wedge \varphi \}$ compute the \DFAs $\A_\xi = (\D_\xi, F_\xi)$.
\item 
Form the product $\D = \D_{\neg\E} \times \D_{\E \rightarrow \varphi} \times \D_{\E \wedge \varphi}$. Lift the final states of each component to the product, i.e. if $\A_\xi = (D_\xi, F_\xi)$ is the \DFA for $\xi$, then the lifted condition $G_\xi$ consists of all states $(s_{\neg \E}, s_{\E \rightarrow \varphi}, s_{\E \wedge \varphi})$ s.t. $s_\xi \in F_\xi$.
\item 
In \DFA game $(\D, G_{\E \rightarrow \varphi})$ compute a uniform positional winning strategy $f_{ag}$. Let $W_{ag} \subseteq S$ be the agent's winning region.
\item 
In \DFA game $(\D, G_{\neg \E})$ compute the environment's winning region $V \subseteq Q$.
\item 
Compute the environment restriction $\D'$ of $\D$ to $V$.
\item 
In \DFA game $(\D', G_{\E \wedge \varphi})$ find a uniform positional cooperatively winning strategy $g_{ag}$.
\item
\textbf{Return} the agent strategy $\sigma_{ag}$ induced by the positional strategy $k_{ag}$, which is defined as follows: 
$k_{ag}(s) = 
\begin{cases}
f_{ag}(s) & \text{ if } s \in W_{ag}, \\
g_{ag}(s) & \text{ otherwise. }
\end{cases}
$
\end{enumerate}

\section{Symbolic \LTLf Best-effort Synthesis~\label{sec:4}}

We present in this section three different symbolic approaches to \LTLf best-effort synthesis, namely monolithic, explicit-compositional, and symbolic-compositional, as depicted in Figure~\ref{fig:algs}. In particular, we base on the symbolic techniques of DFA games presented in~\cite{ZTLPV17}, which we briefly review below.


\begin{figure}
    \centering
    \includegraphics[width=\linewidth]{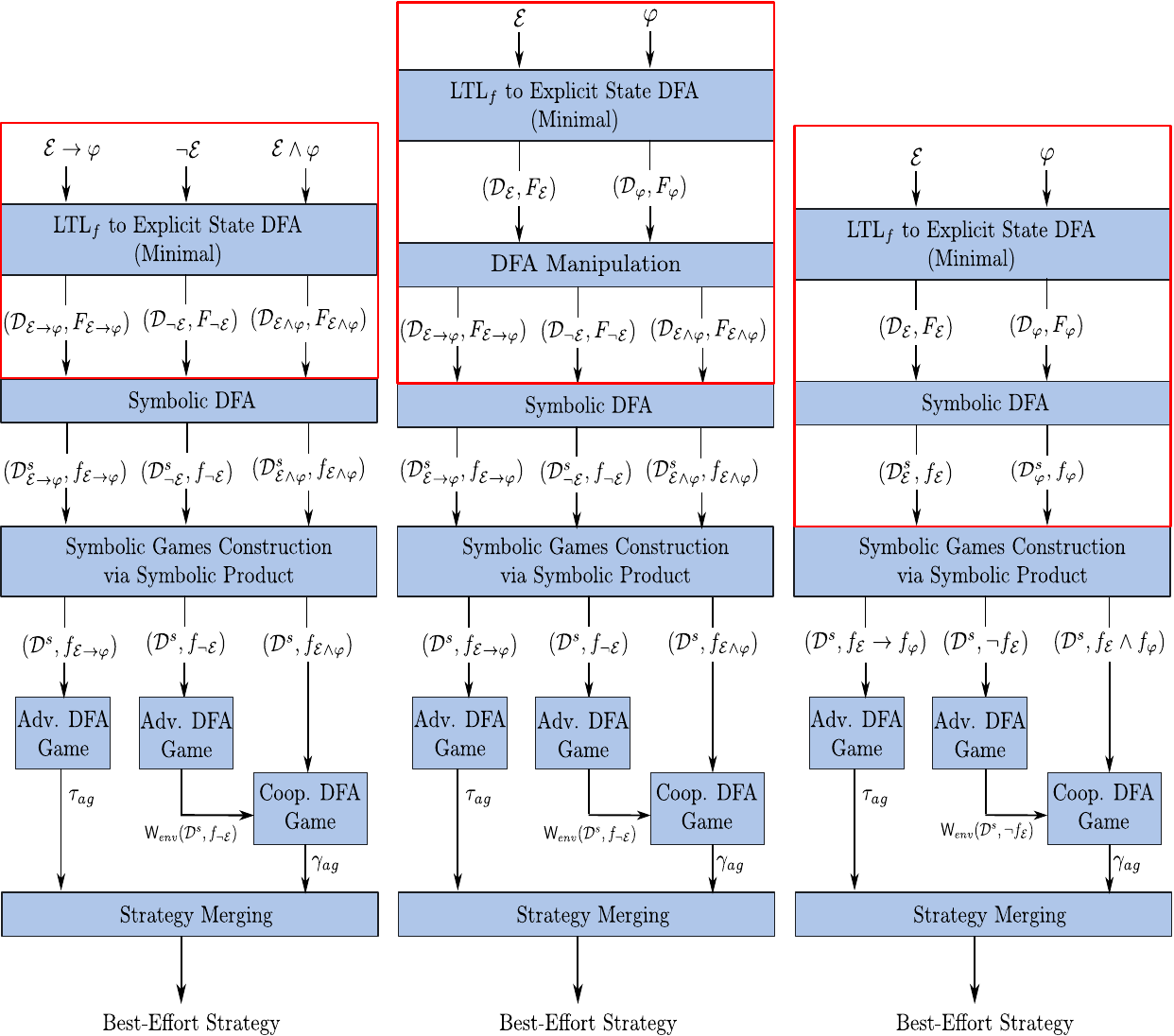}
    \caption{From left to right, \mya monolithic, \myb explicit-compositional, and \myc symbolic-compositional techniques to \LTLf best-effort synthesis. In particular, $\D^s = \D^s_{\E \rightarrow \varphi} \times \D^s_{\neg \E} \times \D^s_{\E \wedge \varphi}$ in \mya and \myb. $\D^s = \D^s_{\E} \times \D^s_{\varphi}$ in \myc. The specific operations of the three techniques are enclosed in red boxes.}
    \label{fig:algs}
\end{figure}

\subsection{Symbolic \DFA Games~\label{sec:4-1}}

We consider the \DFA representation described in Section~\ref{sec:3} as an explicit-state representation. Instead, we are able to represent a \DFA more compactly in a symbolic way by using a logarithmic number of propositions to encode the state space.
More specifically, the \emph{symbolic} representation of $\D$ is a tuple $\D^s = (\X, \Y, \Z, Z_0, \eta)$, where $\Z$ is a set of state variables such that $|\Z| = \lceil \log |S| \rceil$, and every state $s \in S$ corresponds to an interpretation $Z \in 2^\Z$ over $\Z$; $Z_0 \in 2^{\Z}$ is the interpretation corresponding to the initial state $s_0$; $\eta \colon 2^\X \times 2^\Y \times 2^\Z \rightarrow 2^\mathcal{Z}$ is a Boolean function such that $\eta(Z, X, Y) = Z'$ if and only if $Z$ is the interpretation of a state $s$ and $Z'$ is the interpretation of the state $\delta(s, X \cup Y)$.
The set of goal states is represented by a Boolean function $f$ over $\Z$ that is satisfied exactly by the interpretations of states in $F$. In the following, we denote symbolic \DFAs as pairs $(\D^s, f)$.
%


Given a symbolic \DFA game \((\D^s, f)\), we can compute a positional uniform winning agent strategy through a least fixpoint computation over two Boolean formulas $w$ over $\Z$ and $t$ over $\Z \cup \Y$, which represent the agent winning region and winning states with agent actions such that, regardless of how the environment behaves, the agent reaches the final states, respectively. Specifically, $w$ and $t$ are initialized as $w_0(\Z) = f(\Z)$ and $t_0(\Z, \Y) = f(\Z)$, since every goal state is an agent winning state. Note that $t_0$ is independent of the propositions from $\Y$, since once the play reaches goal states, the agent can do whatever it wants. $t_{i+1}$ and $w_{i+1}$ are constructed as follows:

\begin{align*}
    & t_{i+1}(Z, Y) = t_{i}(Z, Y) \vee (\neg w_{i}(Z) \wedge \forall X.w_{i}(\eta(X, Y, Z))) \\
    & w_{i+1}(Z) = \exists Y.t_{i+1}(Z, Y)
\end{align*}

The computation reaches a fixpoint when $w_{i+1} \equiv w_i$. To see why a fixpoint is eventually reached, note that function $w_{i+1}$ is \emph{monotonic}. That is, at each step, a state $Z$ is added to the winning region $w_{i+1}$ only if it has not been already detected as a winning state, written $\neg w_{i}(Z)$ in function $t_{i+1}(Z, Y)$ above, \emph{and} there exists an agent choice $Y$ such that, for every environment choice $X$, the agent moves in $w_{i}$, written $\forall X.w_{i}(\eta(X, Y, Z))$. 

When the fixpoint is reached, no more states will be added, and so all agent winning states have been collected. By evaluating $Z_0$ on $w_{i+1}$ we can determine if there exists a winning strategy. If that is the case, $t_{i+1}$ can be used to compute a uniform positional winning strategy through the mechanism of Boolean synthesis~\cite{FriedTV16}. More specifically, by passing \(t_i\) to a Boolean synthesis procedure, setting \(\Z\) as input variables and \(\Y\) as output variables, we obtain a uniform positional winning strategy \(\tau: 2^\Z \rightarrow 2^\Y\) that can be used to induce an agent winning strategy.

Computing a uniform positional cooperatively winning strategy can be performed through an analogous least-fixpoint computation. To do this, we define again Boolean functions \(\hat{w}\) over \(\Z\) and \(\hat{t}\) over \(\Z \cup \Y\), now representing the agent cooperatively winning region and cooperatively winning states with agent actions such that, if the environment behaves cooperatively, the agent reaches the final states. Analogously, we initialize \(\hat{w}_0(\Z) = f(\Z)\) and \(\hat{t}_0(\Z, \Y) = f(\Z)\).
Then, we construct \(\hat{t}_{i+1}\) and \(\hat{w}_{i+1}\) as follows:

\begin{align*}
    & \hat{t}_{i+1}(Z, Y) = \hat{t}_{i}(Z, Y) \vee (\neg \hat{w}_{i}(Z) \wedge \exists X.\hat{w}_{i}(\eta(X, Y, Z))) \\
    & \hat{w}_{i+1}(Z) = \exists Y.\hat{t}_{i+1}(Z, Y);
\end{align*}

Once the computation reaches the fixpoint, checking the existence and computing a uniform cooperatively winning positional strategy can be done similarly.

Sometimes, the state space of a symbolic transition system must be restricted to not reach a given set of invalid states represented as a Boolean function. To do so, we redirect all transitions from states in the set to a \emph{sink} state. Formally:

\begin{definition}~\label{def:restrict}
    Let $\D^s = (\Z, \X, \Y, Z_0, \eta)$ be a symbolic transition system and $g$ a Boolean formula over $\Z$ representing a set of states. The restriction of $\D^s$ to $g$ is a new symbolic transition system $\D'^s = (\Z, \X, \Y, Z_0, \eta')$, where $\eta'$ only agrees with $\eta$ if $Z \models g$, i.e., $\eta' = \eta \land g$.
\end{definition}



\subsection{Monolithic Approach}\label{sec:4-2}
The monolithic approach is a direct implementation of the best-effort synthesis approach presented in~\cite{ADR2021} (i.e., of Algorithm~0), utilizing the symbolic synthesis framework introduced in~\cite{ZTLPV17}. Given a best-effort synthesis problem $\P = (\E, \varphi)$, we first construct the \DFAs following the synthesis algorithm described in Section~\ref{sec:3}, and convert them into a symbolic representation. Then, we solve suitable games on the symbolic \DFAs and obtain a best-effort strategy. The workflow of the monolithic approach, i.e., \textbf{Algorithm 1}, is shown in Figure~\ref{fig:algs}\mya. We elaborate on the algorithm as follows.

\medskip
\noindent \textbf{Algorithm 1.} Given an \LTLf best-effort synthesis problem $\P = (\E, \varphi)$, proceed as follows:
\begin{enumerate}
    \item 
    For \LTLf formulas \(\E \rightarrow \varphi\), \(\neg \E\) and \(\E \wedge \varphi\) compute the corresponding minimal explicit-state \DFAs \(\A_{\E \rightarrow \varphi} = (\D_{\E \rightarrow \varphi}, F_{\E \rightarrow \varphi})\), \(\A_{\neg \E} = (\D_{\neg \E}, F_{\neg \E})\) and \(\A_{\E \wedge \varphi} = (\D_{\E \wedge \varphi}, F_{\E \wedge \varphi})\).
    \item 
    Convert the \DFAs to a symbolic representation to obtain $\A^s_{\E \rightarrow \varphi} = (\D^s_{\E \rightarrow \varphi}, f_{\E \rightarrow \varphi})$, $\A^s_{\neg \E} = (\D^s_{\neg \E}, f_{\neg \E})$ and $\A^s_{\E \wedge \varphi} = (\D^s_{\E \wedge \varphi}, f_{\E \wedge \varphi})$.
    \item 
    Construct the product $\D^s = \D^s_{\E \limpl \varphi} \times \D^s_{\lneg \E} \times \D^s_{\E \land \varphi}$.
    \item 
    In \DFA game $(\D^s, f_{\E \rightarrow \varphi})$, compute a uniform positional winning strategy $\tau_{\ag}$ and the agent's winning region $\wina(\D^s, f_{\E \rightarrow \varphi})$. 
    \item 
    In \DFA game $(\D^s, f_{\neg \E})$, compute the environment's winning region $\wine(\D^s, f_{\neg \E})$.
    \item 
    Compute the symbolic restriction $\D'^s$ of $\D^s$ to  $\wine(\D^s, f_{\neg \E})$ to restrict the state space of $\D^s$ to considering $\wine(\D^s, f_{\neg \E})$ only.
    \item 
    In \DFA game $(\D'^s, f_{\E \wedge \varphi})$, compute a uniform positional cooperatively winning strategy $\gamma_\ag$.
    \item 
    \textbf{Return} the best-effort strategy $\sigma_{ag}$ \emph{induced} by the positional strategy $\kappa_{ag}$ constructed as follows: $\kappa_{ag}(Z) = 
    \begin{cases}
    \tau_\ag(Z) & \text{ if } Z \models \wina(\D^s, f_{\E \rightarrow \varphi}) \\
    \gamma_\ag(Z) & \text{ otherwise.}
    \end{cases}
    $
\end{enumerate}

The main challenge in the monolithic approach comes from the \LTLf-to-\DFA conversion, which can take, in the worst case, double-exponential time~\cite{DegVa13}, and thus is also considered the bottleneck of \LTLf synthesis~\cite{ZTLPV17}.
To that end, we propose an explicit-compositional approach to diminish this difficulty by decreasing the number of \LTLf-to-\DFA conversions. 

\subsection{Explicit-Compositional Approach}\label{sec:4-3}
As described in Section~\ref{sec:4-2}, the monolithic approach to a best-effort synthesis problem $\P = (\E, \varphi)$ involves three rounds of \LTLf-to-\DFA conversions corresponding to \LTLf formulas $\E \rightarrow \varphi$, $\neg \E$ and $\E \wedge \varphi$. However, observe that \DFAs \(\A_{\E \rightarrow \varphi}\), \(\A_{\neg \E}\) and \(\A_{\E \wedge \varphi}\) can, in fact, be constructed by manipulating the two \DFAs \(\A_{\E}\) and \(\A_{\varphi}\) of \LTLf formulas $\E$ and $\varphi$, respectively. Specifically, given the explicit-state \DFAs \(\A_{\varphi}\) and \(\A_{\E}\), we obtain \(\A_{\E \rightarrow \varphi}\), \(\A_{\neg \E}\) and \(\A_{\E \wedge \varphi}\) as follows:

\begin{compactitem}
    \item \(\A_{\E \rightarrow \varphi} = \comp(\inter(\A_{\E}, \comp(\A_{\varphi}))\);
    \item \(\A_{\neg \E} = \comp(\A_{\E})\);
    \item \(\A_{\E \wedge \varphi} = \inter(\A_{\E}, \A_{\varphi})\);
\end{compactitem}
where $\comp$ and $\inter$ denote complement and intersection on explicit-state \DFAs, respectively. 
Note that transforming \LTLf formulas into \DFAs takes double-exponential time in the size of the formula, while the complement and intersection of \DFAs take polynomial time in the size of the \DFA.

The workflow of the explicit-compositional approach, i.e., \textbf{Algorithm 2}, is shown in Figure~\ref{fig:algs}\myb. 
As the monolithic approach, we first translate the formulas \(\E\) and \(\varphi\) into minimal explicit-state \DFAs \(\A_{\E}\) and \(\A_{\varphi}\), respectively. Then, \DFAs \(\A_{\E \rightarrow \varphi}\), \(\A_{\neg \E}\) and \(\A_{\E \wedge \varphi}\) are constructed by manipulating \(\A_{\E}\) and \(\A_{\varphi}\) through complement and intersection. Indeed, the constructed explicit-state \DFAs are also minimized.
The remaining steps of computing suitable \DFA games are the same as in the monolithic approach.

\subsection{Symbolic-Compositional Approach}\label{sec:4-4}


The monolithic and explicit-compositional approaches are based on playing three games over the symbolic product of transition systems $\D_{\E \limpl \varphi}$, $\D_{\neg \E}$, and $\D_{\E \land \varphi}$. 
We observe that given \DFAs $\A_\E = (\D_{\E}, F_{\E})$ and $\A_{\varphi} = (\D_{\varphi}, F_{\varphi})$ recognizing $\E$ and $\varphi$, respectively, the \DFA recognizing any Boolean combination of $\E$ and $\varphi$ can be constructed by taking the product of $\D_{\E}$ and $\D_{\varphi}$ and properly defining the set of final states over the resulting transition system. 

\begin{lemma}~\label{lem:construct}
    Let $\A_{\psi_1} = (\D_{\psi_1}, F_{\psi_1})$ and $\A_{\psi_2} = (\D_{\psi_2}, F_{\psi_2})$ be the automata recognizing \LTLf formulas $\psi_1$ and $\psi_2$, respectively, and $\psi = \psi_1 \ op \ \psi_2$ denoting an arbitrary Boolean combination of $\psi_1$ and $\psi_2$, i.e., $op \in \{ \land, \lor, \rightarrow, \leftrightarrow\}$. The \DFA $\hat{\A_{\psi}} = (\hat{\D_{\psi}}, \hat{F}_{\psi})$ with $\hat{\D_{\psi}} = \D_{\psi_1} \times \D_{\psi_2}$ and $\hat{F}_{\psi} = \{(s_{\psi_1}, s_{\psi_2}) \ | \ s_{\psi_1} \in F_{\psi_1} \ op \ s_{\psi_2} \in F_{\psi_2}\}$ recognizes $\psi$.
\end{lemma}
\begin{proof}
\noindent ($\rightarrow$) Assume $\pi \models \psi$. We will prove that $\pi \in \L(\hat{\A}_{\varphi})$. To see this, observe that $\pi \models \psi$ implies $\pi \models \psi_1 \ op \ \pi \models \psi_2$. It follows by \cite{DegVa13} that $\pi \in \L(\A_{\psi_1}) \ op \ \pi \in \L(\A_{\psi_2})$, meaning that running $\pi$ in $\D_{\psi_1}$ and $\D_{\psi_2}$ yields the sequences of states $(s_0^{\psi_1}, \ldots, s_n^{\psi_1})$ and $(s_0^{\psi_2}, \ldots, s_n^{\psi_2})$ such that  $s_n^{\psi_1} \in F_{\psi_1} \ op \ s_n^{\psi_2} \in F_{\psi_2}$. Since $\hat{\D}_\psi$ is obtained through synchronous product of $\D_{\psi_1}$ and $\D_{\psi_2}$, running $\pi$ in $\hat{\A}_{\psi}$ yields the sequence of states $((s_0^{\psi_1}, s_0^{\psi_2}), \ldots, (s_n^{\psi_1}, s_n^{\psi_2}))$, such that $(s_n^{\psi_1}, s_n^{\psi_2}) \in \hat{F}_{\psi}$. Hence, we have that $\pi \in \L(\hat{\A}_{\psi})$. \newline
\noindent ($\leftarrow$) Assume $\pi \in \L(\hat{\A}_{\varphi})$. We prove that $\pi \models \psi$. To see this, observe that $\pi \in \L(\hat{\A}_{\varphi})$ means that the run $\rho = (s_0^{\psi_1}, s_0^{\psi_2}) \ldots (s_n^{\psi_1}, s_n^{\psi_2})$ induced by $\pi$ on $\hat{\D}_{\psi}$ is such that $(s_n^{\psi_1}, s_n^{\psi_2}) \in \hat{F}_{\psi}$. This means, by construction of $\hat{F}_{\psi}$, that $(s_n^{\psi_1}, s_n^{\psi_2}) \text { s.t. } s_n^{\psi_1} \in F_{\psi_1} \ op \ s_n^{\psi_2} \in F_{\psi_2}$. Since $\hat{\D}_{\psi}$ is obtained through synchronous product of $\D_{\psi_1}$ and $\D_{\psi_2}$, it follows that $\pi \in \L(\A_{\psi_1}) \ op \ \pi \in \L(\A_{\psi_2})$. By \cite{DegVa13} we have that $\pi \models \psi_1 \ op \ \pi \models \psi_2$, and hence $\pi \models \psi$. 
\end{proof}

Notably, Lemma~\ref{lem:construct} tells that the \DFAs $\A_{\E \limpl \varphi}$, $\A_{\neg \E}$, and $\A_{\E \land \varphi}$ can be constructed from the same transition system by defining proper sets of final states. Specifically, given the \DFAs $\A_{\E} = (\D_{\E}, F_{\E})$ and $\A_{\varphi} = (\D_{\varphi}, F_{\varphi})$ recognizing $\E$ and $\varphi$, respectively, the \DFAs recognizing $\E \limpl \varphi$, $\lneg \E$, and $\E \land \varphi$ can be constructed as $\A_{\E \limpl \varphi} = (\D, F_{\E \limpl \varphi})$, $\A_{\lneg \E} = (\D, F_{\lneg \E})$, and $\A_{\E \land \varphi} = (\D, F_{\E \land \varphi})$, respectively, where $\D = \D_{\E} \times \D_{\varphi}$ and: 
\begin{compactitem}
    \item $F_{\E \limpl \varphi} = \{(s_{\E}, s_{\varphi}) \mid s_{\E} \in F_{\E} \limpl s_{\varphi} \in F_{\varphi}\}$.
    \item $F_{\lneg \E} = \{(s_{\E}, s_{\varphi}) \mid s_{\E} \not \in F_{\E}\}$.
    \item $F_{\E \land \varphi} = \{(s_{\E}, s_{\varphi}) \mid s_{\E} \in F_{\E} \land s_{\varphi} \in F_{\varphi}\}$.
\end{compactitem}

The symbolic-compositional approach precisely bases on this observation. As shown in Figure~\ref{fig:algs}\myc, we first transform the \LTLf formulas \(\E\) and \(\varphi\) into minimal explicit-state \DFAs \(\A_{\E}\) and \(\A_{\varphi}\), respectively, and then construct the symbolic representations \(\A^s_{\E}\) and \(\A^s_{\varphi}\) of them.
Subsequently, we construct the symbolic product $\D^s = \D^s_{\E} \times \D^s_{\varphi}$, once and for all, and get the three \DFA games by defining the final states (which are Boolean functions) from \(f_{\E}\) and \(f_{\varphi}\) as follows:
\begin{compactitem}
    \item \(f_{\E \rightarrow \varphi} = f_{\E} \rightarrow f_{\varphi}\).
    \item \(f_{\neg \E} = \neg f_{\E}\).
    \item \(f_{\E \wedge \varphi} = f_{\E} \wedge f_{\varphi}\).
\end{compactitem}
From now on, the remaining steps are the same as in the monolithic and explicit-compositional approaches.

\medskip
\noindent \textbf{Algorithm 3.} Given a best-effort synthesis problem \(\P = (\E, \varphi)\), proceed as follows:
\begin{enumerate}
\item 
    Compute the minimal explicit-state \DFAs \(\A_{\E} = (\D_{\E}, F_{\E})\) and \(\A_{\varphi} = (\D_{\varphi}, F_{\varphi})\).
    \item 
    Convert the \DFAs to a symbolic representation to obtain \(\A^s_{\E} = (\D^s_{\E}, f_{\E})\) and \(\A^s_{\varphi} = (\D^s_{\varphi}, f_{\varphi})\).
\item Construct the symbolic product $\D^s = \D^s_{\E} \times \D^s_{\varphi}$.
\item 
    In \DFA game $\G^s_{\E \rightarrow \varphi} = (\D^s, f_{\E} \limpl f_{\varphi})$ compute a positional uniform winning strategy $\tau_{ag}$ and the agent winning region $\wina(\D^s, f_{\E} \limpl f_{\varphi})$. 
\item 
    In the \DFA game $(\D^s, \neg f_{\E})$ compute the environment's winning region $\wine(\D^s, \neg f_{\E})$.
\item 
    Compute the symbolic restriction $\D'^s$ of $\D^s$ to  $\wine(\D^s, f_{\neg \E})$ so as to restrict the state space of $\D^s$ to considering $\wine(\D^s, f_{\neg \E})$ only.
\item 
    In the \DFA game $(\D'^s, f_{\E} \wedge f_{\varphi})$ find a positional cooperatively winning strategy \(\gamma_{ag}\).
\item 
\textbf{Return} the best-effort strategy $\sigma_{ag}$ \emph{induced} by the positional strategy $\kappa_{ag}$ constructed as follows: $\kappa_{ag}(Z) = 
    \begin{cases}
    \tau_\ag(Z) & \text{ if } Z \models \wina(\D^s, f_{\E \rightarrow \varphi}) \\
    \gamma_\ag(Z) & \text{ otherwise.}
    \end{cases}
    $
\end{enumerate}



\section{Empirical Evaluations}\label{sec:5}
In this section, we first describe how we implemented our symbolic \LTLf best-effort synthesis approaches described in Section~\ref{sec:4}. Then, by empirical evaluation, we show that Algorithm 3, i.e., the symbolic-compositional approach, shows an overall best-performance. In particular, we show that performing best-effort synthesis only brings a minimal overhead with respect to standard synthesis and may even show better performance on certain instances.
\subsection{Implementation}

We implemented the three symbolic approaches  to \LTLf best-effort synthesis described in Section~\ref{sec:4} in a tool called \textit{BeSyft}, by extending the symbolic synthesis framework~\cite{ZTLPV17,TabajaraV19} integrated in state-of-the-art synthesis tools~\cite{BLTV,DeGiacomoF21}. In particular, we based on \textsc{Lydia}~\footnote{https://github.com/whitemech/lydia}, the overall best performing \LTLf-to-\DFA conversion tool, to construct the minimal explicit-state \DFAs of \LTLf formulas. 
Moreover, \besyft borrows the rich APIs from \textsc{Lydia} to perform relevant explicit-state \DFA manipulations required by both Algorithm 1, i.e., the monolithic approach~(c.f.,  Subsection~\ref{sec:4-2}), and Algorithm 2, i.e., the explicit-compositional approach~(c.f., Subsection~\ref{sec:4-3}), such as complement, intersection, minimization. As in~\cite{ZTLPV17,TabajaraV19}, the symbolic \DFA games are represented in Binary Decision Diagrams~(BDDs)~\cite{Bryant92}, utilizing CUDD-3.0.0~\cite{cudd} as the BDD library. Thereby, \besyft constructs and solves symbolic \DFA games using Boolean operations provided by CUDD-3.0.0, such as negation, conjunction, and quantification. 
The uniform positional winning strategy $\tau_\ag$ and the uniform positional cooperatively winning strategy $\gamma_\ag$ are computed utilizing Boolean synthesis~\cite{FriedTV16}. The positional best-effort strategy is obtained by applying suitable Boolean operations on $\tau_\ag$ and $\gamma_\ag$. 
As a result, we have three derivations of \besyft, namely \besyft-Alg-1, \besyft-Alg-2, and \besyft-Alg-3, corresponding to the monolithic, explicit-compositional, and symbolic-compositional approach, respectively.

\subsection{Experiment Methodology} 

\paragraph{Experiment Setup.}
All experiments were run on a laptop with an operating system 64-bit Ubuntu 20.04, 3.6 GHz CPU, and 12 GB of memory. Time out
was set to 1000 seconds.

\paragraph{Benchmarks.} We devised a \emph{counter-game} benchmark, based on the one proposed in \cite{ZGPV20}. More specifically, there is an $n$-bit binary counter and, at each round, the environment chooses whether to issue an increment request for the counter or not. The agent can choose to grant the request or ignore it and its goal is to get the counter to have all bits set to $1$. The increment requests only come from the environment, and occur in accordance with the environment specification. The size of the minimal \DFA of a counter-game specification grows exponentially as $n$ increases. 


In the experiments, environment specifications ensure that the environment eventually issues a minimum number $K$ of increment requests in sequence, which can be represented as \LTLf formulas $\E_{K} = \Diamond(add \wedge \Wnext(add) \ldots \wedge \Wnext(\ldots(\Wnext(add))\ldots))$, where $K$ is the number of conjuncts. Counter-game instances may be realizable depending on the parameter $K$ and the number of bits $n$. In the case of a realizable instance, a strategy for the agent to enforce the goal is to grant all increment requests coming from the environment. Else, the agent can achieve the goal only if the environment behaves cooperatively, such as issuing more increment requests than that specified in the environment specification. That is, the agent needs a best-effort strategy. In our experiments, we considered counter-game instances with at most $n=10$ bits and $K=10$ sequential increment requests. As a result, our benchmarks consist of a total of $100$ instances.

\subsection{Experimental Results and Analysis.} 

\begin{figure}
    \centering
    \includegraphics[scale = .45]{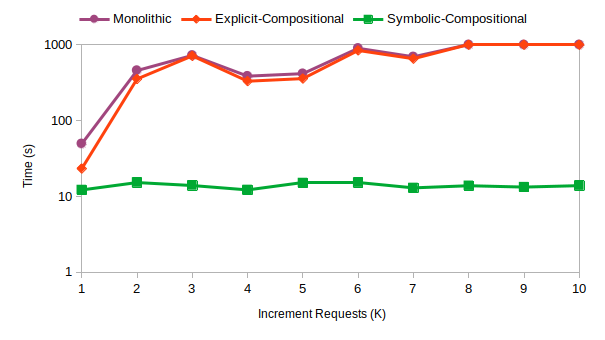}
    \caption{Comparison (in log scale) of \besyft implementations on counter game instances with $n=8$ and $1 \leq K \leq 10$.}
    \label{fig:besyft-comprson}
\end{figure}

\begin{figure}
    \centering
    \includegraphics[scale = .45]{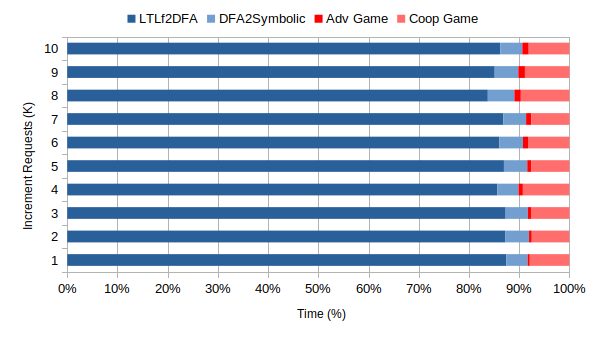}
    \caption{Relative time cost of \besyft-Alg-3 major operations on counter game instances with $n=8$ and $1 \leq K \leq 10$.}
    \label{fig:besyft-runtimes}
\end{figure}


In our experiments, all \besyft implementations are only able to solve counter-game instances with up to $n=8$ bits. Figure~\ref{fig:besyft-comprson} shows the comparison (in log scale) of the three symbolic implementations of best-effort synthesis on counter-game instances with $n=8$ and $1 \leq K \leq 10$. First, we observe that \besyft-Alg-1~(monolithic) and \besyft-Alg-2~(explicit-compositional) reach timeout when $K \geq 8$, whereas \besyft-Alg-3~(symbolic-compositional) is able to solve all $8$-bit counter-game instances. We can also see that \besyft-Alg-1 performs worse than the other two derivations since it requires three rounds of \LTLf-to-\DFA conversions, which in the worst case, can lead to a double-exponential blowup. Finally, we note that \besyft-Alg-3, which implements the symbolic-compositional approach, achieves orders of magnitude better performance than the other two implementations, although it does not fully exploit the power of \DFA minimization. Nevertheless, it is not the case that automata minimization always leads to improvement. Instead, there is a tread-off of performing automata minimization. As shown in Figure~\ref{fig:besyft-comprson}, \besyft-Alg-3, performs better than \besyft-Alg-2, though the former does not minimize the game arena after the symbolic product, and the latter minimizes the game arena as much as possible.

On a closer inspection, we evaluated the time cost of each major operation of \besyft-Alg-3, and present the results on counter-game instances with $n = 8$ and $1 \leq K \leq 10$ in Figure~\ref{fig:besyft-runtimes}. First, the results show that \LTLf-to-\DFA conversion is the bottleneck of \LTLf best-effort synthesis, the cost of which dominates the total running time. 
Furthermore, we can see that the total time cost of solving the cooperative \DFA game counts for less than $10 \%$ of the total time cost. As a result, we conclude that performing best-effort synthesis only brings a minimal overhead with respect to standard reactive synthesis, which consists of constructing the \DFA of the input \LTLf formula and solving its corresponding adversarial game. Also, we observe that solving the cooperative game takes longer than solving the adversarial game. Indeed, this is because the fixpoint computation in the cooperative game often requires more iterations than that in the adversarial game. 

\begin{figure}
    \centering
    \includegraphics[scale = .45]{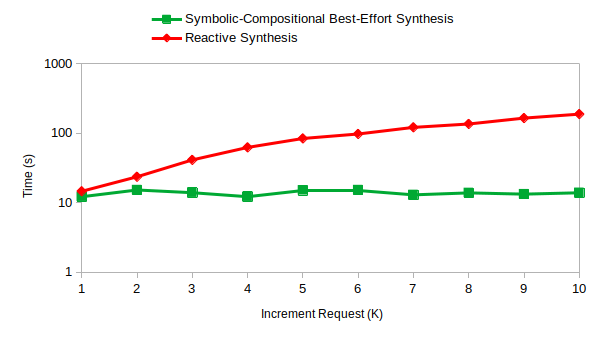}
    \caption{Comparison (in log scale) of \besyft-Alg-3 and implementations of symbolic \LTLf reactive synthesis on counter-game instances with $n=8$ and $1 \leq K \leq 10$.}
    \label{fig:bes-vs-rs}
\end{figure}

Finally, we also compared the time cost of symbolic-compositional best-effort synthesis with that of standard reactive synthesis on counter-game instances. More specifically, we considered a symbolic implementation of reactive synthesis that computes an agent strategy that enforces the \LTLf formula $\E \limpl \varphi$~\cite{lydia,ZTLPV17}, which can be used to find an agent strategy enforcing $\varphi$ under $\E$, if it exists~\cite{ADMR19}. Interestingly, Figure~\ref{fig:bes-vs-rs} shows that for certain counter-game instances, symbolic-compositional best-effort synthesis takes even less time than standard reactive synthesis. It should be noted that symbolic-compositional best-effort synthesis performs \LTLf-to-\DFA conversions of \LTLf formulas $\varphi$ and $\E$ separately and combines them to obtain the final game arena without having automata minimization, whereas reactive synthesis performs the \LTLf-to-\DFA conversion of formula $\E \limpl \varphi$ and minimizes its corresponding \DFA. These results confirm the practical feasibility of best-effort synthesis and that automata minimization does not always guarantee performance improvement.

\section{Conclusion~\label{sec:6}}
 We presented three different symbolic \LTLf best-effort synthesis approaches: monolithic, explicit-compositional, and symbolic-compositional. Empirical evaluations proved the outperformance of the symbolic-compositional approach. An interesting observation is that, although previous studies suggest taking the maximal advantage of automata minimization~\cite{TabajaraV19,ZGPV20}, in the case of \LTLf best-effort synthesis, there can be a trade-off in doing so. Another significant finding is that the best-performing \LTLf best-effort synthesis approach only brings a minimal overhead compared to standard synthesis. Given this nice computational result, a natural future direction would be looking into 
\LTLf best-effort synthesis with multiple environment assumptions~\cite{AminofGLMR21}. 

\section*{Acknowledgments}

This work has been partially supported by the ERC-ADG White- Mech (No. 834228), the EU ICT-48 2020 project TAILOR (No. 952215), the PRIN project RIPER (No. 20203FFYLK), and the PNRR MUR project FAIR (No. PE0000013).

\bibliographystyle{splncs04}
\bibliography{eumas23}



\end{document}